\documentclass[runningheads]{llncs}
\usepackage{graphicx, color}
\usepackage{amssymb, amsmath}
\usepackage[ruled,vlined,lined,algochapter]{algorithm2e}
\newtheorem{teo}{Theorem}
\newtheorem{coro}{Corollary}
\newtheorem{lema}{Lemma}
\newtheorem{prop}{Proposition}
\newtheorem{ejem}{Example}
\newtheorem{defi}{Definition}

\newcommand{\ds}{\displaystyle}
\newcommand\Omit[1]{}
\newcommand{\ssi}{\Leftrightarrow}

\newcommand{\partes}[1]{{\mathcal P}({#1})}

\newcommand{\A}{\mathcal{A}}
\newcommand{\R}{\mathcal{R}}
\newcommand{\Pref}{\mathcal{P}}
\newcommand{\F}{\mathcal{F}}
\newcommand{\G}{\mathcal{G}}
\newcommand{\mirec}{\unrhd}

\newcommand{\dom}{\succ}

\newcommand{\miplr}{\succeq_{_\R}}
\newcommand{\mplr}{\succ_{_\R}}
\newcommand{\iplr}{\simeq_{_\R}}

\newcommand{\hmi}{\succeq_{h}}
\newcommand{\hm}{\succ_{h}}
\newcommand{\hi}{\sim_{h}}

\def\qed{\hfill {\scriptsize{$\blacksquare$}}}

\begin{document}
\title{Resource allocation under uncertainty: an algebraic and qualitative treatment\thanks{The third author is supported by the Project CDCHTA-ULA C-1855-13-05-A.}}
\titlerunning{Resource allocation under uncertainty}
%
\author{Franklin Camacho\inst{1}
\and
Gerardo Chac\'on\inst{2}%
\and
Ram\'on Pino Per\'ez\inst{3}
%
}
\authorrunning{F. Camacho et al.}
%
\institute{School of Mathematical Sciences and Information Techonology,
Yachay Tech University, Urcuqu\'{i}, Ecuador.\\ \email{fcamacho@yachaytech.edu.ec}
 \and
Department of Science, Technology and Mathematics, Gallaudet University, Washington DC, USA.\\
\email{gerardo.chacon@gallaudet.edu}\and
Departamento de Matem\'aticas, Facultad de Ciencias, Universidad de Los Andes - M\'erida, Venezuela.\\
\email{pino@ula.ve}}
\maketitle              
\begin{abstract}

We use  an algebraic viewpoint, namely a matrix framework to deal with the problem of resource allocation under uncertainty in the context of
a qualitative approach. 
Our basic qualitative data are
a plausibility relation over the resources, a hierarchical  relation over the agents and of course the preference that the agents have over the resources. With this data we propose a qualitative binary relation $\unrhd$ between allocations such that $\mathcal{F}\unrhd \mathcal{G}$ has the  following intended meaning: the allocation $\mathcal{F}$ produces more or equal social welfare than the allocation $\mathcal{G}$.
We prove that there is a family of  allocations  which are maximal with respect to $\unrhd$. We prove also that there is a notion of simple deal such that  optimal allocations can be reached by sequences of simple deals. Finally, we introduce some mechanism for discriminating {optimal} allocations.
\end{abstract}
\keywords{Resource allocation \and Optimal allocations \and Social welfare \and Qualitative framework}
\section{Introduction}\label{intro}

The task of distribution of a divisible or indivisible set of goods among a set of agents  has been studied mainly in Economics.  The mathematical models of fair distributions, in particular
when the goods are divisible, have been studied by some mathematicians and economists such as Banach and Steinhaus \cite{Ste48,Ste49} before the middle of the Twentieth century and by Brams, Taylor, Zwicker and others towards the end of the past century \cite{BTZ94,BT95,BTZ97} (for a good survey see \cite{BT96}). Some axiomatic analysis about resource allocation problems have been done by Moulin, Thomson and BeviÃÂ¡ among others\cite{MT97,Tho11,Bev96}.
In more recent years the mathematical models of satisfactory distributions have occupied  different communities of researchers besides economists.  Among them, computer scientists working in the area of Artificial Intelligence and Multi-agent Systems, see for instance the works of Sandholm \cite{San98,San99}, Endriss et al. \cite{EMST06} and Chevaleyre et al. \cite{CEEM07,CEM17}.

The problem  concerning  this paper is the distribution of a finite set of indivisible goods (resources) into a finite set of individuals (agents). This is known as the problem of resource allocation. There are many concerns about this problem. The first one is how to consider a determined allocation satisfactory. To deal with this aspect some concepts of social welfare
are used, most of them coming from aggregation of numerical functions of individual utility \cite{Mou88,EMST06}. Thus,  the notion of optimal allocation  arises naturally as the one  producing the maximal social welfare.

Another important aspect studied is the rationality of negotiation between agents in order to reach better resource allocations. Again here the notion of {\em being better} is closely connected to the improvement of the social welfare \cite{San98,EMST06}. From the algorithmic point of view it is important that these negotiations be simple and that any sequence of these simple negotiations
ends in an optimal allocation.

In this work we continue an investigation of these problems  in which we don't use numbers. That is, we propose  a qualitative notion of preference between the allocations that is based only
on the knowledge of the set of resources that agents desire, a plausibility relation between set of resources, and a hierarchical relation between agents. In the works  of Pino P\'erez and colleagues \cite{PV12,PVC16}  three strong hypothesis were made.
First, it was assumed that there was a linear order over the resources. Second, that the plausibility relation between the resources was the possibilistic relation associated to this linear order. Third, it was supposed that there was a linear order between agents encoding a hierarchy between
 them. In this work we generalize these three assumptions. We suppose now that the relation between the resources and the relation between
 the agents are total preorders. We suppose also that the plausibility relation between the sets of resources is a lifting
 (an extension, see \cite{BBP04}) of the relation between the resources with very few properties.

With these tools  we use the {\em dominance plausible rule} -a sort of qualitative expected utility introduced by Dubois et al. \cite{DFPP02,DFP03}- in order to compare two allocations. We show with easy matricial representations that   there is a family of optimal allocations. We propose a notion of {\em simple rational deal} and we prove that we can reach  optimal allocations
by sequences of simple rational deals. 

This work is structured as follows. Section \ref{preli} contains the basic notions. Section \ref{classifying} contains the notion of preference between allocations based on a plausibility relation between sets of resources {which in turn is} based on a  qualitative notion of social welfare. Section \ref{deals} is devoted to the notion of simple rational deals. Section~\ref{good allocations} is devoted to define good resource allocations and to see that they can be reached
by sequences of simple rational deals. In Section~\ref{optimal}  we see that good resource allocations are, indeed, optimal with respect to the notion of qualitative welfare.
Section~\ref{discriminating} is devoted to a method to discriminate  optimal good allocations.
We finish with Section \ref{conclu} in which we make some concluding remarks and present some lines of future work.

\section{Preliminaries}\label{preli}
In this section, we show the preliminary concepts and notation that we will use throughout the article.
Given a nonempty set $X$  and a binary relation $\succeq$ defined  on $X$. \textit{The strict relation } $\succ$ associated to $\succeq$ is  defined as 
\begin{equation}\label{eq-relation-stric}
a\succ b\ssi [(a\succeq b) \,\wedge \, \neg(b\succeq a)],
\end{equation}
the \textit{ indifferent relation} $\backsimeq$ is defined as \begin{equation}\label{eq-relation-indif}
a \backsimeq b\ssi [(a\succeq b) \,\wedge \, (b\succeq a)].
\end{equation}

If $A\subseteq X$, we denote the maximal elements of $A$ as $\max(A)$ and analogously we denote
the minimal elements of $A$ as $\min(A)$

If $X$ is finite, $X=\{x_1,\cdots,x_n\}$, the relation $\succeq$ defines a \textit{relation matrix  }  as:\begin{equation}\label{eq-matriz-relation}
X_{ij}=\left\{
                                          \begin{array}{ll}
                                                           1, & \hbox{if}\, x_i\succeq x_j; \\
                                                           0, & \hbox{otherwise.}
                                                         \end{array}
                                                   \right.
\end{equation}

Note that the previous matrix depends upon the order of the elements of $X$. But any reordering of the elements of $X$ has an associated relation matrix with respect to $\succeq$ closely related to the first matrix. More precisely, if  $(X_{ij})_{n\times n}$ and $(Y_{ij})_{n\times n}$ are matrices associated to $\succeq$ with two different orders of $X$, then there exists a permutation $P$, of the identity matrix  $I_n$ of orden $n$, such that $$(Y_{ij})_{n\times n}=P(X_{ij})_{n\times n}$$

A relation is a total pre-order on $X$ if it is total and transitive. A relation is a linear order if it is irreflexive and transitive.

We start by the defining the two main sets to be studied. These are the set of agents and the set of resources.
We will denote the set of agents by $\A$ which is defined as a finite, nonempty set with $|\A|=q$, i.e,  $\A=\{1,2,\cdots,q\}$.
We will assume there exists a total pre-order $\hmi$ defined on $\A$. 
Given two agents $i,j\in \A$, the symbol $i\hmi j$ will be interpreted as: \emph{agent $i$ has a higher or equal hierarchy than agent $j$}.
The strict relation $\hm$ associated to $\hmi$  defined as  \eqref{eq-relation-stric} can be thought as: \textit{agent $i$ has greater hierarchy than agent $j$}.
The indifferent relation $\hi$ defined as in \eqref{eq-relation-indif} can be interpreted as:  \textit{agent  $i$ has the same hierarchy as agent $j$}.
Sometimes we will abuse the notation and denote by $i\in \A$ an agent and at the same time a natural number between 1 and $q$.
A relation matrix associated to $\hmi$,  defined as  in equation \eqref{eq-matriz-relation} will be  called \emph{a hierarchy matrix} and denoted by
$\A_h=\ds(a_{ij})_{q\times q}$.
The set of \textit{resources} will be denoted as $\R=\{r_1,\cdots,r_k\}$. We will assume that a total pre-order $\miplr$, is defined on $\R$.
The strict and indifferent relations $\mplr$ and $\iplr$, are defined in similar way as in equations \eqref{eq-relation-stric} and \eqref{eq-relation-indif}, respectively.
Based on the relation  $\miplr$ we will later define a  binary relation $\sqsupseteq$ on the set of subsets of $\R$, denoted as $\mathcal{P}(\R)$. Such relation should be a \emph{lifting} of $\miplr$, in the sense that it should preserve the order $\miplr$ when restricted to the the singletons.  Such relation will be understood as a `plausibility' relation on $\partes \R$.


Once defined the sets $\A$ and $\R$, we will assume that each agent is interested in some resources. 
This relation can be recorded in the following  \emph{matrix of request}  $\Pref=\ds(P_{in})_{q\times k}$ that contains the information about the resources each agent  requests.
\begin{equation}\label{eq-matriz-de-preferencia}
P_{in}=\left\{
                                                         \begin{array}{ll}
                                                           1, & \hbox{if}\,\, \hbox{agent}\, i\, \hbox{requests the resource }r_n; \\
                                                           0, & \hbox{otherwise.}
                                                         \end{array}
                                                       \right.
\end{equation}
\begin{remark}\label{remark-notacion-prefe}
Sometimes, we will refer to a resource $r\in\R$ without specifying the subindex. In such cases, $P_{ir}$ {denotes} the entry of the matrix $\Pref$ in the $i$-th row and the column associated to the resource $r$.
\end{remark}

We make the assumptions that the agents either request or do not request a resource, that this is known beforehand, and that it is invariant. Also, we assume that an agent does not have different levels of preference over the resources. This is a subject that could be addressed in future research. Under this conditions, we notice that the $i-$th row of $\Pref$ contains complete information about the resources that agent $i$ requests, whereas the $n-$th column determines all the agents that request the resource $r_n$. 



We are now ready to define the main concept of the article. A \textit{resource allocation} is a distribution of the resources among the agents where each resource is granted to one and only one agent. We introduce a matrix notation   that will be used throughout the rest of the article.
\begin{defi}
  Let $\A$ and $\R$ be nonempty sets, with $k=|\R|$ and $q=|\A|$. A resource allocation  is a matrix 
  $$\F=(f_{in})_{q\times k}$$ such that for each   $1\leq n \leq k$, there exists a unique $1\leq i\leq q$ for which  $f_{in}=1$ and $f_{jn}=0$ for $j\neq i$.
\end{defi}
Notice that if $f_{in}=1$ then it is understood that agent $i$ is the only one who was granted the resource $r_n$.


Similarly to the notation used with the matrix of requests, we will write $f_{ir}$ to denote the entry of the matrix $\F$ corresponding to the $i$-th row and the column associated to the resource $r$.

Sometimes we will use a set of resource allocations  $\F_1,\dots, \F_m$. In such case, we will use the notation $\F_l=(f^l_{in})_{q\times k}$ when necessary.\\

Given a fixed resource $r\in\R$, we define a ``priority relation'' on $\A$ associated to $r$ in the following way: \textit{agent $i$ has a higher priority than agent $j$} if and only if \textit{agent $i$ has equal or grater hierarchy than agent   $j$ and agent $i $ requests the resource $r$, or agent $i$ does not have equal or greater hierarchy than agent  $j$ and agent $j$ does not request the resource $r$.}

\begin{defi}
Given $r\in \R$,  \emph{the priority matrix associated to} $r$, is the matrix  $(a^{r}_{ij})_{q\times q}$, where
\begin{equation}\label{eq-matriz-prioridad}
  a^{r}_{ij}=a_{ij}P_{ir}+(1-a_{ij})(1-P_{jr})
\end{equation}
\end{defi}

Notice that each term of equation \eqref{eq-matriz-prioridad} is either one or zero. If $a_{ij}P_{ir}=1$ then agent $i$ has a higher hierarchy than agent $j$ and agent $i$ requests resource $r$. On the other hand, if $(1-a_{ij})(1-P_{jr})=1$, then $a_{ij}=0$ and $P_{jr}=0$. This implies  agent $i$ doest not have a higher hierarchy than agent $j$ and agent $j$ does not request resource $r$.
This definition captures that on \cite{PVC16} were the hierarchy order is assumed to be a linear order.



\section{Classifying allocations}\label{classifying}


As a first step in classifying resource allocations, we want to  define  an equivalence relation among allocations.  Then we will define the set where one resource allocation  dominates  another.




\begin{defi}\label{def-colu-equiv}
Let $r\in \mathcal{R}$
and let $c$ and $c^*$ be two canonical vectors of size $q$. Suppose that  $i, j\in \A$ are such that the $i$-th entry of $c$ is equal to 1 and the $j$-th entry of $c^*$ is equal to 1.   We will say that $c$ and $c^*$ are $r$-equivalent, denoted as $c\equiv_{r} c^*$ if  agents $i$ and $j$ have the same hierarchy and the same preference over the resource $r$. In other words,

\begin{equation}\label{eq-vect-equi}
c\equiv_{r} c^*\ssi (a_{ij}=a_{ji}=1) \wedge (P_{ir}=P_{jr})
\end{equation}

\end{defi}


Let  $\F$ and $\G$ be two resource allocations and define the set of all non-$r$-equivalent columns as
\begin{equation}\label{eq-col-no equi}
D_{\F\G}=\{r\in \R: c^{\F}_r \not\equiv_{r}  c^{\G}_{r}\}
\end{equation} where $c^{\F}_r$ and $c^{\G}_r$ denote the columns of $\F$ and $\G$, respectively,  associated to the resource $r$. Notice that $D_{\F\G}=D_{\G\F}$.


We will define a set where the resource allocation $\F$ \emph{dominates} the resource allocation $\G$. Roughly, this will be the set of resources that are assigned under $\F$ to an agent with  higher priority than if assigned under $\G$.
\begin{defi}\label{conjunto-dominante}
  Let $\F=(f_{in})$ and $\G=(g_{in})$ be two resource allocations. The set $[\F\succ \G]$ where $\F$ dominates $\G$, is defined as:
\begin{equation}\label{eq-conjunto-dominante}
[\F\succ \G]=\{r\in D_{\F\G}: a^{r}_{ij}=1,\, \mbox{where}\,f_{i r}=1\,\mbox{and}\,\, g_{j r}=1 \}
\end{equation}
\end{defi}

In other words,  $r\in [\F\succ \G]$ if and only if for the unique agents $i$ and $j$ such that $ f_{ir}=1$ and  $ g_{jr}=1$ it holds that $ a^{r}_{ij}=1$ and the columns $(f_{ir})_{1\leq i\leq q}$ and $(g_{ir})_{1\leq i\leq q}$ are not $r$-equivalent.
.

Once we have  defined  dominant sets as subsets of $\R$, we will define a  way of comparing resource allocations.
The idea is to capture a notion of social welfare \cite{Mou88,Mou03,EMST06} used to compare allocations via a notion of qualitative welfare.
The notion we use next is inspired of the {\em dominance plausible rule} proposed by Dubois el al. \cite{DFPP02,DFP03}.

\begin{defi}
Let $\sqsupseteq$  a plausibility relation on the set $\mathcal{P}(\R)$. We say that the resource allocation $\F$ produces more social welfare than a resource allocation $\G$, denoted as $\F\unrhd \G$, if
\begin{equation}
[\F\succ \G]\sqsupseteq[\G\succ \F].
\end{equation}
Moreover, we will say that $\F\rhd \G$ if $\F\unrhd \G$ and $\neg(\G \unrhd \F)$.
\end{defi}

Notice that the previous definition depends on a binary relation on $\mathcal{P}(\R)$ that, for now, could be arbitrary. A similar approach was studied in \cite{PVC16} where a strict possibilistic relation was considered for the case in which $\miplr$ was a  linear order.
Indeed our relation $\sqsupseteq$ will be an extension (see \cite{BBP04} for a study of the ways to do extensions and their properties) of the relation $\miplr$ having very few properties.

We are ready to define the notion of optimality among resource allocations. Notice that such definition will depend on the binary relation $\sqsupseteq$.

\begin{defi}\label{defi_optimal}
Given a binary relation  $\sqsupseteq$ on  the set  $\mathcal{P}(\R)$. We will say that a resource allocation $\F^*$ is optimal if $\F^* \mirec \G$ for every resource allocation $\G$.
\end{defi}

Here it is relevant to say that optimal resource allocations need not to be unique since the previous definition does not distinguish among equivalent allocations.


\section{Simple  and rational deals}\label{deals}

In order to search for an optimal resource allocation, we will first  introduce  negotiation concepts.
By a deal we will understand a pair of allocations that are equal except for only one column. We will define the concept in a more rigorous way by making use of permutation matrices.

Denote as  $I_q$ the identity matrix of size $q\times  q$. A permutation matrix $E_{ij}$ is the result of interchanging the  rows $i$ and $j$ of $I_q$.
Clearly, $E_{ij}=E_{ji}$ .

\begin{defi}\label{def-trato-simple}
Given two resource allocations $\F$ and $\G$ such that $\F\neq\G$, we say that the pair $(\F,\G)$ is a simple deal if there exists a unique $r\in\R$ and a permutation matrix $E_{ij}\neq I_q$, such that if  $\F=[ c_1,\cdots,c_{r-1},c_r,c_{r+1}\cdots c_k]$, where $c_1,\dots, c_k$ denote the columns of $F$, then
$$\G=[ c_1,\cdots,c_{r-1},E_{ij}c_r,c_{r+1}\cdots c_k].$$
\end{defi}

\begin{lema} If  $(\F,\G)$ is a simple deal, then the permutation matrix $E_{ij}$ is unique.

\end{lema}
\begin{proof}
Suppose that $\F=(f_{in})_{q\times k}$, and $\G=(g_{in})_{q\times k}$ are two resource allocations such that $(\F,\G)$ is a simple deal, then there exists a unique resource $r$ such that $\F$ and $\G$ differ only in the column associated to resource $r$. Let $i\in \A$ and $j\in\A$ be such that $f_{ir}=1$ and $g_{jr}=1$, then $f_{i^*r}=0$ for all $i^*\neq i$ and $g_{j^*r}=0$ for all $j^*\neq j$.

Now if $E_{i^*j^*}$, is another permutation matrix, with $i^*\neq i$ or $j^*\neq j$, then either $f_{i^*r}=0$ or  $g_{j^*r}=0$  and consequently $E_{i^*j^*}c_r\neq  E_{ij}c_r$, where $c_r$ is the $r$-th column of $\F$.
\end{proof}

Using the notation of the previous lemma, we will say that the permutation matrix $E_{ij}$ \emph{represents} the simple deal $(\F,\G)$ associated to the resource $r\in\R$.



\begin{defi}\label{def-trato-racional-simple}
Given two resource allocations $\F$ and $\G$, we will say that $(\F,\G)$ is a rational deal if it is a simple deal represented by the permutation matrix $E_{ij}$, and if one of the following conditions holds.

\begin{equation}\label{condicion1-trato-racional}
  (a_{ji}=1)\wedge (P_{jr}=1)
\end{equation}
\noindent or
\begin{equation}\label{condicion2-trato-racional}
 (a_{ij}=1) \wedge  (\forall t,\, (a_{t i}=1\Rightarrow P_{tr}=0))
\end{equation}
\end{defi}



%

Condition \eqref{condicion1-trato-racional} means that agent $j$ has a higher or equal hierarchy than $i$ and moreover, requests resource $r$.
  Condition \eqref{condicion2-trato-racional} says that agent $i$  has a higher or equal hierarchy than $j$, that $i$ does not request the resource $r$ and no other agent $t$ with higher or equal hierarchy as $i$ requests the resource $r$.


\begin{defi}
A sequence of resource allocations $\F_1$,$\F_2,\dots$,$\F_m$ is called \emph{sequence of rational deals}  if for every $1\leq i\leq m-1$, the pair $(\F_i, \F_{i+1})$ is a rational deal.
\end{defi}

The following lemma and  the corollary that follows will be  needed in Section 6.


\begin{lema}\label{lema-matriz-permut-racional}
  Suppose that the permutation matrices $E_{ij}$ y $E_{i^*j^*}$ represent, respectively,  the rational deals $(\F_1,\F_2)$ and $(\F_2,\F_3)$ associated to the same resource $r\in\R$. Then the product $E_{ij^*}= E_{i^*j^*}E_{ij}$ also represents a rational deal associated to $r$.
\end{lema}

\begin{proof}
  We will show that the matrix $E=E_{i^*j^*}E_{ij}$ represents the rational deal $(\F_1,\F_3)$ associated to $r$.
First, notice that $i^*=j$ and  $E=E_{ij^*}$.

Suppose that $\F_1=[c_1,\dots,c_r,\dots,c_k]$, then since $(\F_1,\F_2)$ and $(\F_2,\F_3)$ are simple deals, we can write
\[\F_2=[c_1,\dots,E_{ij}c_r,\dots,c_k]\]  and  \[\F_3=[c_1,\dots,E_{i^*j^*} E_{ij}c_r,\dots,c_k]=[c_1,\dots,E_{ij^*} c_r,\dots,c_k].\] In consequence, $(\F_1,\F_3)$ is a simple deal.

%
%
%

It remains to show that one of the following conditions hold:
\begin{equation}\label{eq-condicion1-lema-permu-trato-racional}
  (a_{j^*i}=1)\wedge (P_{j^*r}=1)
\end{equation}
or
\begin{equation}\label{eq-condicion2-lema-permu-trato-racional}
 (a_{ij^*}=1)  \wedge (\forall t,\, (a_{t i}=1\Rightarrow P_{t r}=0)).
\end{equation}
 But since $(\F_1,\F_2)$ and $(\F_2,\F_3)$  are rational deals, we have
\begin{enumerate}
  \item $E_{ij}$ represents $(\F_1,\F_2)$ and one of the following conditions hold:
\begin{enumerate}
\item\label{i-condicion1-trato-1-lema} $(a_{ji}=1)\,\wedge \,(P_{j r}=1)$, or
\item\label{ii-condicion2-trato-1-lema}
 $(a_{ij}=1)\,  \wedge \,(\forall t,\, (a_{t i}=1\Rightarrow P_{t r}=0))$.
\end{enumerate}
  \item  $E_{i^*j^*}$ represents $(\F_2,\F_3)$  and one of the following conditions hold:
\begin{enumerate}
\item\label{a-condicion1-trato-1-lema} $(a_{j^*i^*}=1)\,\wedge \, (P_{j^*r}=1)$, or
\item\label{b-condicion2-trato-1-lema} $ (a_{i^*j^*}=1)   \wedge (\forall t,\, (a_{t i^*}=1\Rightarrow P_{t r}=0))$.
\end{enumerate}
\end{enumerate}

If (\ref{i-condicion1-trato-1-lema}) and (\ref{a-condicion1-trato-1-lema}) hold, then    $a_{ji}=1$ and $a_{j^*i^*}=1$ and since the relation $\hmi$ is transitive and $i^*=j$, then $a_{j^*i}=1$. Moreover, from (\ref{a-condicion1-trato-1-lema}), we have that $P_{j^*r}=1$ and consequently (\ref{eq-condicion1-lema-permu-trato-racional}) holds. \\
Now suppose that  (\ref{ii-condicion2-trato-1-lema}) and (\ref{b-condicion2-trato-1-lema}) hold. Again by transitivity of $\hmi$ we have that $a_{ij^*}=1$.
Moreover, condition (\ref{b-condicion2-trato-1-lema}) implies (\ref{eq-condicion2-lema-permu-trato-racional}).\\
If (\ref{i-condicion1-trato-1-lema}) and (\ref{b-condicion2-trato-1-lema}) hold, then
$P_{j r}=1$  and $P_{i^* r}=P_{j r}=0$ which is a contradiction.\\
Similarly, if  (\ref{ii-condicion2-trato-1-lema}) and (\ref{a-condicion1-trato-1-lema}) hold, then $a_{j^*i}=1$ and $ P_{j^* r}=1$ and taking $t=j^*$ in (\ref{a-condicion1-trato-1-lema}), we have that $P_{j^* r}=0$. Again a contradiction.\\
Hence, neither conditions (\ref{i-condicion1-trato-1-lema}) and (\ref{b-condicion2-trato-1-lema}) nor conditions (\ref{ii-condicion2-trato-1-lema}) and (\ref{a-condicion1-trato-1-lema}) occur at the same time. This finishes the proof.


\end{proof}

The previous lemma can be generalized using mathematical induction to obtain the following result.

\begin{coro}\label{coro-matrices-perm-racional}
  Suppose that the permutation matrices $E_{i_1j_1}, E_{i_2j_2},\dots,$ $E_{i_mj_m}$ represent consecutive rational deals associated to the same resource $r$. Then  $j_1=i_2$, $j_2=i_3,\cdots, j_{m-1}=i_m$, and $E_{i_1j_m}=E_{i_mj_m}\cdots E_{i_2j_2} E_{i_1j_1}$ is a permutation matrix that also represents a rational deal associated to the resource $r$.
\end{coro}

\section{Good  allocations}\label{good allocations}

Our goal will be to construct an algorithm to find optimal resource allocations. Before, we will define what ``good'' resource allocations are.

\begin{defi}\label{def-col-y-asig-buenas}
  Let $\F=(f_{i n})_{q\times k}$ be a resource allocation and fix $r\in\R$. Suppose  $i\in \A$ is the only agent such that $f_{i r}=1$.
We will say that the column $c_r=(f_{i r})_{1 \leq i\leq q}$, associated to the resource $r$, is in \emph{good position} if
\begin{equation}\label{defi-posicio-optima-cond1}
  (P_{i r}=1)\wedge(\forall t ( a_{t i}=1\wedge a_{i t}=0)\Rightarrow P_{tr}=0))
\end{equation}
or
\begin{equation}\label{defi-posicio-optima-cond2}
  (\forall t  (P_{t r}=0))\wedge (i\in\min(\A))
\end{equation}
where $\min(\A)=\{t\in \A: a_{it}=1, \forall i\in \A\}$.

We will say the $\F$ is a \emph{good resource allocation} if all its columns are in good position.
\end{defi}

Condition  (\ref{defi-posicio-optima-cond1})  can be interpreted as: Agent $i$ {requests}  the resource $r$ and every other agent with higher hierarchy does not request it.  Condition (\ref{defi-posicio-optima-cond2}) says that agent $i$ has the lowest possible hierarchy and no agents {requests} the resource $r$. A good resource allocation will then be an allocation in which each resource is either assigned to an agent that  {requests} it and has the higher hierarchy among those that {request} the resource, or to an agent of the lowest hierarchy if no agent {requests} the resource.


%
The following lemma says that a column that is in good position cannot be further ``improved'' by using a rational deal.

\begin{lema}
  Suppose that the $r$-th column $c_r$ of a resource allocation $\mathcal{F}$ is in good position and that there exists a resource allocation $\G$ such that $(\F,\G)$ is a rational deal associated to $c_r$, represented by the matrix $E_{ij}$. Then $c_r$ and $E_{ij}c_r$ are equivalent.
\end{lema}
\begin{proof}
Suppose $\F=[c_1,\dots,c_r,\cdots,c_k]$, we will prove that $a_{ij}=1=a_{ji}$ and $P_{ir}=P_{jr}$.  Since $\hmi$ is total we consider the following two cases:

If  $a_{ji}=1$, then condition (\ref{condicion1-trato-racional}) holds and consequently we have that $P_{jr}=1$  and condition  \eqref{defi-posicio-optima-cond2} does not hold.
But since $c_r$ is in good position, then \eqref{defi-posicio-optima-cond1} holds and consequently $P_{ir}=1 $ and  $a_{ij}=1$. Hence $c_r$ and $E_{ij}c_r$ are $r$-equivalent.

Condition  $a_{ij}=1$ and $a_{ji}= 0$ is impossible since if this were the case, then condition (\ref{condicion2-trato-racional}) holds and consequently  $P_{ir}=0$. Moreover $ i\notin\min(A)$.  This two facts contradict conditions   (\ref{defi-posicio-optima-cond1}) and (\ref{defi-posicio-optima-cond2}), respectively. This finishes the proof.
\end{proof}

%
The following lemma says that if the column involved in a simple deal turns out to be in good position, then the deal must be rational.

\begin{lema}\label{lema-trat-simpl+opt-impl-racional}
Suppose $\F$ and $\G$ are two resource allocations and that $(\F,\G)$ is a simple deal that converts column $c_r$ to $E_{ij}c_r$.  If $E_{ij}c_r$ is in good position and it is not equivalent to $c_r$, then $(\F,\G)$ is a rational deal.
\end{lema}

\begin{proof}
We need to show that one of the conditions \eqref{condicion1-trato-racional} or \eqref{condicion2-trato-racional} hold. We will use again the totality of $\hmi$ and consider cases. Suppose that $a_{j i}=1$, then we have that $j\not\in\min(\A)$ since otherwise  $a_{i j}=1$ and  by condition \eqref{defi-posicio-optima-cond2} $P_{i r}=0=P_{j r}$, contradicting the hypothesis that $E_{ij}c_r$ and $c_r$ are not equivalent.
Hence, condition \eqref{defi-posicio-optima-cond1} holds and consequently $P_{jr}=1$ and \eqref{condicion1-trato-racional} holds.

Now suppose that $a_{i j}=1$ and $a_{j i}=0$,  if condition \eqref{defi-posicio-optima-cond2} holds, then in particular we have that condition \eqref{condicion2-trato-racional} holds. If on the contrary, condition   \eqref{defi-posicio-optima-cond1} holds, then $P_{jr}=1$ and $P_{ir}=0$  and moreover, for every $t\in\A$ such that $a_{ti}=1$ we have by transitivity that $a_{tj}=1$ and $a_{jt}=0$ which, by condition \eqref{defi-posicio-optima-cond1} implies that $P_{tn}=0$ proving that condition \eqref{condicion2-trato-racional} holds.
\end{proof}

\begin{remark}
The converse of the previous lemma is not generally true. That is, the resulting column involved in a rational deal, needs not to be in good position.
\end{remark}

\begin{lema}\label{lema-not-opt-impl-deal-rational-good} If $\F=[c_1,\cdots,c_r,\cdots c_k]$ is resource allocation such that
 the column $c_r$ is not good position, then  there exists a rational deal which
transforms it in a good column.
\end{lema}

\begin{proof}
Let $c_r$ the column in not good position. We want to find a resource allocation $\G$ such that $(\F,\G)$ is a rational deal associated to the resource in the position  $r$ and that leaves $c_{r}$ in good position. \\
Let $i\in \A$ be the only agent such that $f_{ir}=1$.
Define  $$I_{r}=\{t\in \A: t\neq i\, \wedge \,P_{tr}=1 \}$$
 If $I_{r}=\emptyset$, let $j$ be any element in $ \min(\A)$. If $I_{r}\neq\emptyset$, let $j$ be any element in $ \max(I_{r})$. \\
Now,
if $\F=(f_{in})_{q\times k}=[c_1,\cdots,c_{r}\cdots, c_k]$, let's define  the resource allocation $\G=[c_1,\cdots,E_{ij}c_{r}\cdots, c_k]$.
 Clearly $(\F,\G)$ is a simple deal. To show that it is actually a rational deal, we will show that $E_{ij}c_{r}$ is in good position, and since it is not equivalent to $c_r$, and then use lemma \ref{lema-trat-simpl+opt-impl-racional}. \\
We will consider each case, $I_{r}=\emptyset$ and $I_{r}\neq\emptyset$, separately:\\
Suppose first that $I_{r}=\emptyset$, then $\forall t\in \A$, $t\neq i$ we have that $P_{tr}=0$. Thus $P_{ir}=0$ since otherwise condition
 \eqref{defi-posicio-optima-cond1} would hold and $c_{r}$ would be in good position, contradicting the hypothesis.
Consequently,  for every $t\in \A$, $P_{tr}=0$ and since $j\in \min(A)$, then condition \eqref{defi-posicio-optima-cond2} holds. Hence, $E_{ij}c_{r}$ is in good position.\\
Now suppose that  $I_{r}\neq \emptyset$, then $P_{jr}=1$ and $j\in \max(I_{r})$. Consequently, if $t\in\A$  is such that $t\neq i$, $a_{tj}=1$ and $a_{jt}=0$, we have that  $j\notin I_{r}$ and therefore $P_{tr}=0$.
 Moreover, it can not happen that $a_{ij}=1$ and $ a_{ji}=0$ and $ P_{ir}=1$ since the maximality of $j$ will then imply that $P_{tr}=0$ for every $t\in\A$ such that $a_{ti}=1$ but this implies that condition  \eqref{defi-posicio-optima-cond1} holds for $i$ reaching a contradiction.
Thus, condition \eqref{defi-posicio-optima-cond1} holds for $j$ and hence $E_{ij}c_{r}$ is in good position.
\end{proof}

The following theorem shows that any resource allocation can be converted into a good resource allocation by means of a finite sequence of rational deals.  The proof of the theorem gives us an algorithm that will later result in  an optimal resource allocation as we will see in Corollary \ref{coro:optimal=buena}.

\begin{teo}\label{thm_good_resource} Suppose $\F$ is not a good resource allocation and let $l$ be the number of columns of $\F$ that are not in good position. Then there exists a sequence of rational deals   $\F_1,\cdots, \F_{l+1}$ such that $\F_1=\F$ and  $\F_{l+1}$ is a good allocation.
\end{teo}
\begin{proof}
Let $\mathcal{C}=\{c_{n_1},\cdots,c_{n_l}\}$, $1\leq n_1<\cdots<n_l\leq k$,  the set of columns that are not in good position of $\F$. Obviously, $l\geq 1$.
For $c_{n_1}$,
define  $$\F_1=\F=(c_1,\cdots,c_{n_1},\cdots,c_k).$$
By Lemma \ref{lema-not-opt-impl-deal-rational-good}, there  exists a rational deal $(\F_1,\F_2)$ associated to $c_{n_1}$ such that $$\F_2=(c_1,\cdots,E_{i_1j_1}c_{n_1},\cdots,c_k)$$
and $E_{i_1j_1}c_{n_1}$ is in a good position. If $l=1$, then $\F_2$ is a good allocation and this finishes the proof.
\\
Now, suppose that $l> 1$.
{Using the same reasoning, construct the resource allocations $\F_3,\dots,\F_{l+1}$ until a good resource allocation is reached.}\qed
\end{proof}

\begin{algorithm}[]
\textbf{Input} ($\F$, $\A_h$, $\Pref$) \hspace{0.3cm} \{$\F$ allcation, $\A_h$ hierarchy matrix, $\Pref$ requests matrix,
$k$ number of resources, and  $q$ number of agents\}\\
$\F^*= \F$\\
$m\in\min(A_h)$\\
$r=1 $\\
\While {$r\leq k$}{$Cr= (f_{tr})_{1\leq t\leq q};$ \hspace{.3cm}\{column of resource in the position  $r$\}\\
$i=$ ``position of $Cr$ iqual to 1'';\\
$I(r)=\{t:t\neq i,\, \wedge\, P_{tr}=1\}$\\
$w=\max(I(r))$\\
\If
{
$P_{ir=1}$ \textbf{and} $\left[\displaystyle\sum_{t=1,t\neq i}^{q}a_{ti}(1-a_{it})(1-P_{tr})=\displaystyle\sum_{t=1,t\neq i}^{q}a_{ti}(1-a_{it})\right]$\\
\textbf{or} \\
$\displaystyle\sum_{t=1}^{q}(P_{tr}+(1-a_{it}))=0$
}
{
$Cr$ is a good position\\
$r=r+1$
}
\Else{
\If{$\displaystyle\sum_{t=1}^{q}P_{tr}=0$}{$j=m$}\Else{$j=w$}}
$Cr=E_{ij}Cr$\hspace{.3cm} \{$E_{ij}$ permutation matrix \}\\
$\F=\F^*$\\
$r=r+1$
}
\textbf{Exit}($\F^*$)
\caption{Algorithm for a good allocation}
\end{algorithm}

\section{Optimal resource allocations}\label{optimal}

In this subsection we will show that good resource allocations are actually optimal under the Definition \ref{defi_optimal} as long as the relation $\sqsupseteq$ is positive. First we need two technical lemmas.



\begin{lema}\label{lema-conjuntos-domi-0}
Let $\F$ and $\G$ be two resource allocations, then $$D_{\F \G}=[\F\dom \G]\cup [\G\dom \F]\quad\text{ and }\quad [\F\dom \G]\cap [\G\dom \F]=\emptyset$$
\end{lema}

\begin{proof}
If $D_{\F \G}=\emptyset$, then $[\F\dom \G]=\emptyset= [\G\dom \F]$.
Suppose that  $D_{\F \G}\neq\emptyset$. It is clear that $[\F\dom \G]\cup [\G\dom \F]\subseteq D_{\F \G} $.
It remains to show that $D_{\F \G}\subseteq [\F\dom \G]\cup [\G\dom \F]$.
Let $r\in D_{\F \G}$ and let $i,j\in\A$ be the only agents such that $f_{i r}=1= g_{j r}$.

Suppose first that $a_{ij}\neq a_{ji}$, then from equation \eqref{eq-matriz-prioridad} we have that either
$$a_{ij}^r=P_{ir}\quad\mbox{and }\quad a_{ji}^r=1-P_{ir}$$
or
$$a_{ij}^r=1-P_{jr}\quad\mbox{and}\quad a_{ji}^r=P_{jr}.$$
In any of these cases, either $a_{ij}^r=1$ or  $a_{ji}^r=1$. Thus,
 $r\in [\F\dom \G] \cup[\G\dom \F] $.\\
Suppose now that $a_{ij}= a_{ji}$. Then since $r\in D_{\F \G}$ we have that  $P_{ir}\neq P_{ir}$. Again by equation \eqref{eq-matriz-prioridad} we have that either
$$a_{ij}^r=P_{ir}\quad\mbox{and}\quad a_{ji}^r=P_{jr}$$
or
$$a_{ij}^r=1-P_{jr}\quad\mbox{and}\quad a_{ji}^r=1-P_{ir}.$$
And since $P_{ir}\neq P_{jr}$, the previous statement implies that $r\in [\F\dom \G] $ or $r\in [\G\dom \F] $.
Hence, $D_{\F \G}\subseteq [\F\dom \G]\cup [\G\dom \F]$.

On the other hand, suppose that $r\in [\F\dom \G]\cap [\G\dom \F]$. Then $r\in D_{\F \G}$ and there exist agents $i,j\in\A$ such that $f_{ir}=1=g_{jr}$, and   $a^r_{ij}=1=a^r_{ji}$.

From equation  \eqref{eq-matriz-prioridad}, exactly one of the following two possibilities  happens:
$$a_{ij}P_{ir}=1=a_{ji}P_{jr}$$
or
$$(1-a_{ij})(1-P_{jr})=1=(1-a_{ji})(1-P_{ir})$$

In both cases, $a_{ij}=a_{ji}$ and $P_{ir}=P_{jr}$ contradicting the hypothesis that $c^l_r\not\equiv_r c^p_r$. Hence, $[\F\dom \G]\cap [\G\dom \F]=\emptyset$.

\end{proof}
\Omit{
\begin{lema}\label{lema-conjuntos-domi-1}
Suppose $\F$ and $\G$ are two resource allocations. Then $$[\F\dom \G]\cap [\G\dom \F]=\emptyset$$
\end{lema}}
\Omit{
\begin{proof}
Suppose that $r\in [\F\dom \G]\cap [\G\dom \F]$. Then $r\in D_{\F \G}$ and there exist agents $i,j\in\A$ such that $f_{ir}=1=g_{jr}$, and   $a^r_{ij}=1=a^r_{ji}$.

From equation  \eqref{eq-matriz-prioridad}, exactly one of the following two possibilities  happens:
$$a_{ij}P_{ir}=1=a_{ji}P_{jr}$$
or
$$(1-a_{ij})(1-P_{jr})=1=(1-a_{ji})(1-P_{ir})$$

In both cases, $a_{ij}=a_{ji}$ and $P_{ir}=P_{jr}$ contradicting the hypothesis that $\textcolor{red}{c^l_r\not\equiv_r c^p_r}$. Hence, $[\F\dom \G]\cap [\G\dom \F]=\emptyset$.
\end{proof}
}

\begin{prop}\label{prop:sequence_D}
  Let $\F_1$,$\F_1,\dots$,$\F_n$  be a sequence of rational deals. Then for every $1< p\leq n$, we have that $D_{1,p}=[\F_p\dom \F_1]$,
  where $D_{1,p}$ is an abbreviation of for $D_{\F_1\F_p}$.
\end{prop}
\begin{proof}
Let $1<p\leq n$ be fixed. Clearly,
\begin{equation*}
[\F_p\dom \F_1]\subset D_{1,p}
\end{equation*}

It remains to show that $ D_{1,p} \subset[\F_p\dom \F_1]$.  Let $r\in D_{1,p}$ and let $i,j\in\A$ be the only agents such that $f^1_{ir}=1=f^p_{jr}$. We will show that
$$a_{ji}^r=1=a_{ji}P_{jr}+ (1-a_{ji})(1-P_{ir})$$

Let $l\leq p$ be the number of rational deals from the sequence $\F_1$,$\F_2,\dots$,$\F_p$ associated to the resource $r$.
Let $\G_{1},\,\G_{2},\cdots \,\G_{{l+1}}$ be the subsequence of rational deals of $\F_1$,$\F_2,\dots$,$\F_p$ that involve only the resource $r$.  $\G_1=\F_1$ and $\G_l=\F_p$.
Let $E_{i_1j_1}, E_{i_2j_2},\cdots, E_{i_lj_l}$ be the permutation matrices that represent the rational deals $(\G_{1},\G_{2}),\cdots,$ $ (\G_{l},\G_{l+1})$, respectively. Notice that $i=i_1$, $j_1=i_2$,\ $j_2=i_3, \cdots, j_{l-1}=i_l $,  and $j_{l}=j$.
Using corollary \ref{coro-matrices-perm-racional}, we get that $$E_{ij}=E_{i_lj_l}\cdots E_{i_2j_2}E_{i_1j_1}$$
represents a rational deal associated to $r$.  Thus, by Definition \ref{def-trato-racional-simple}, it holds that either
\begin{equation}\label{eq1-aux-prop}
a_{ji}P_{jr}=1
\end{equation}

or
\begin{equation}\label{eq2-aux-prop}
(a_{ij}=1)\wedge (P_{ir}=0)\wedge (\forall k(a_{ki}=1\Rightarrow P_{kr}=0)).
\end{equation}

If equation \eqref{eq1-aux-prop} holds, then $a^r_{ji}=a_{ji}P_{jr}=1$ and consequently $r\in [\F_p\dom \F_1]$.

Let's suppose equation \eqref{eq2-aux-prop} holds. Since $c^1_r\not \equiv_r c^p_r$ then either
$a_{ij}\neq a_{ji}$ or $a_{ij}=a_{ji}$ and $P_{ir}\neq P_{jr} $.

If $a_{ij}\neq a_{ji}$ then $ a_{ji}=0$ and since $P_{ir}=0$, then
$(1-a_{ji})(1-P_{ir})=1=a^r_{ji}.$

On the other hand, if $a_{ij}=a_{ji}$ and $P_{ir}\neq P_{jr} $, then $a_{ji}=1$ and $P_{ir}=1$ and consequently, $a^r_{ji}=a_{ji}P_{jr}=1$.
In both cases we have that $r\in [\F_p\dom \F_1]$.

Thus $D_{1p}\subset [\F_p\dom \F_1]$ and this finishes the proof.
\end{proof}


\begin{lema}\label{lema-asig-buenas-son-optimas}
Suppose $\G$ is a good resource allocation and $\F$ is any resource allocation. Then $[\G\dom \F]= D_{\G\F}$
\end{lema}

\begin{proof}
It is enough to show that  $D_{\G,\F}\subset [\G\dom \F] $.
Let $r\in D_{\G,\F}$ and suppose that $i,k\in \A$ are the only agents such that  $g_{ir}=1$ and $f_{kr}=1$.

We want to show that  $a^r_{ik}=1$ where
\begin{equation}\label{matriz-prioridad-auxiliar}
a^r_{ik}=a_{ik}P_{ir}+ (1-a_{ik})(1-P_{kr})
\end{equation}

\begin{description}
\item[Case 1:] $P_{ir}=1$. Then either $a_{ik}=1$ or $a_{ki}=0$.

If $a_{ik}=1$, then pluging in equation \eqref{matriz-prioridad-auxiliar}, we have that $a^r_{ik}=1$.

If $a_{ik}=0$, then since $\hmi$ is total, then $a_{ki}=1$.
Now, since the any column of the allocation $\G$ is in good position, then condition and $P_{ir}=1$, then \eqref{defi-posicio-optima-cond1} holds:
\[(P_{ir}=1)\wedge(\forall j(a_{ji}=1\wedge a_{ij}=0)\Rightarrow P_{jm}=0 ).\]
Thus, $a_{ki}=1$ and $a_{ik}=0$ imply that $P_{kr}=0$. Plugging in  equation \eqref{matriz-prioridad-auxiliar}, we conclude that $a^r_{ik}=1$.

\item[Case 2:] $P_{ir}=0$. Again using the hypothesis that any column of $\G$ is in good position, condition  \eqref{defi-posicio-optima-cond2} holds:
\[\forall j (P_{jr}=0)\wedge (i\in\min(\A)).\]
In consequence, $P_{kr}=0$ and $a_{ki}=1$. Hence  $P_{kr}=P_{ir}$ and since the $r$ columns in $\F$ and $\G$ are not equivalent, then we must have that $a_{ik}=0$. Pluging in equation \ref{matriz-prioridad-auxiliar},  we obtain $a^r_{ik}=1$.
\end{description}
This finishes the proof.
\end{proof}

The next corollary follows from Lemmas \ref{lema-conjuntos-domi-0} and \ref{lema-asig-buenas-son-optimas}.

\begin{coro}\label{coro_empty}
Suppose $\G$ is a good resource allocation, then for every resource allocation $\F$ we have that $[\F\dom \G]= \emptyset$.
\end{coro}

We are finally ready to prove that good resource allocations and optimal resource allocations coincide when the binary relation $\sqsupseteq$ is positive in the following sense:

\begin{defi}
Given a binary relation $\sqsupseteq$ defined on $\mathcal{P}(\R)$, we say that $\sqsupseteq$ is \emph{positive} if $\emptyset \sqsupseteq \emptyset$ and for every set $A\in\mathcal{P}(\R)$, with $A\neq\emptyset$ we have that \[A \sqsupseteq \emptyset \,\wedge\, \emptyset\not\sqsupseteq A.\]
\end{defi}

\begin{teo}\label{coro:optimal=buena}
Suppose $\sqsupseteq$ is positive. Then any good resource allocation is also an optimal resource allocation and vice versa.
\end{teo}
\begin{proof}
Let $\G$ be a good resource allocation and $\F$ be any resource allocation, then by Lemma \ref{lema-asig-buenas-son-optimas} and Corollary \ref{coro_empty} we have that \[[\G\dom \F]= D_{\G\F}\sqsupseteq\emptyset = [\F\dom \G]\] and consequently $\G\unrhd\F$. Thus $\G$ is optimal.

Conversely, suppose $\G$ not a good resource allocation, then by Theorem \ref{thm_good_resource} and Proposition \ref{prop:sequence_D}, there exists a good resource allocation $\F$ such that $[\F\dom\G]=D_{\F\G}$, consequently $[\F\dom\G]=\emptyset$ and by the positivity of $\sqsupseteq$,  we have that $\F\rhd\G$. Thus $\G$ is not optimal.
\end{proof}

\begin{remark}
For any two optimal resource allocations $\F$ and $\G$ it holds that $\F\unrhd\G$ and $\G\unrhd\F$.
\end{remark}

\begin{coro}\label{coro:existe_mejor_optima}
Suppose $\sqsupseteq$ is positive, then for every resource allocation $\F$ there exists an optimal resource allocation $\F^*$ such that $\F^*\mirec\F$.
\end{coro}
\Omit{
\begin{proof}
If $\F$ is an optimal resource allocation, the result holds trivially. Suppose $\F$ is not optimal, then by Theorem \ref{thm_good_resource} it is possible to find a sequence $\F_1,\F_2,\dots, \F_l$ of rational deals such that $\F_1=\F$ and $\F^*=\F_l$ is a good resource allocation. Now by Corollary \ref{coro_empty} we have that $[\F\dom \F^*]= \emptyset$ and since the relation $\sqsupseteq$ is nonnegative, we have that $[\F\dom \F^*]\sqsupseteq [\F^*\dom \F]$. Thus, $\F^*\unrhd \F$ and the result holds by Corollary \ref{coro:optimal=buena}.
\end{proof}
}
The following corollary justifies the term ``optimal allocation'' in the sense of the relation $\unrhd$.
\begin{coro}
Suppose $\sqsupseteq$ is positive. Suppose also that $\G$ is an  optimal resource allocation and $\F$ is a non-optimal resource allocation. Then $\G\rhd \F$.
\end{coro}

\begin{proof}
Corollary \ref{coro_empty} implies that $[\F\dom \G]= \emptyset$ and the result follows from the positivity of  $\sqsupseteq$.
\end{proof}

\Omit{
\begin{proof}
This is a direct consequence of the previous Corollary.
\end{proof}
\textcolor{red}{NUEVO}
}
\section{Discriminating good allocations}\label{discriminating}
In this section we search to discriminate good allocations. We use a criterion of local satisfaction to compare the the good allocations the relation optimal with respect this criterion will be the the allocation locally fair.

In order to understand the problems of lack of discrimination let us see the following example:

\begin{ejem}\label{example-1}
Let $A=\{1,2,3\}$ and $\R=\{r_1, r_2, r_3, r_4, r_5, r_6\}$. Suppose that all the agents have equal hierarchy; moreover
 Agent 3 requests all the resources except excepto $r_6$;  Agent 2 requests the resources $r_1$, $r_2$, $r_3$ and $r_4$  but he doesn't request $r_5$ nor $r_6$, whereas agente 1, request the first 3 resources.
Then the request matrix $\Pref$ and the hierarchy matrix $\A_h $  are
$$\Pref=
\left(
\begin{array}{c|cccccc}
& r_1 & r_2& r_3& r_4& r_5& r_6\\
\hline
1&1 & 1& 1& 0& 0&0\\
2&1 & 1& 1& 1& 0&0\\
3&1 & 1& 1& 1& 1&0\\
\end{array}
\right)\quad A_{h}=\left(\begin{matrix}
1&1&1\\
1&1&1\\
1&1&1\\
\end{matrix}\right)
$$

Notice that we can make a partition of $\R$ in the subsets $\{r_1,\,r_2\,r_3\}$, $\{r_4\}$, $\{r_5\}$ and $\{r_6\}$. Any
 good allocation distributes the resources $\{r_1,\,r_2,\,r_3\}$ among the agents $\{1,2,3\}$; $\{r_4\}$ among the agents $\{2,3 \}$; the resource $\{r_5\}$ to agent 5 and the resource $\{r_6\}$ among the agents $\{1,2,3\}$. Now consider the following allocations:
\begin{center}
{\footnotesize
$\mathcal{E}=\left(\begin{matrix}
0&0&0&0&0&0\\
0&0&0&0&0&0\\
1&1&1&1&1&1\\
\end{matrix}\right)$,
$\F=\left(\begin{matrix}
1&1&0&0&0&0\\
0&0&1&1&0&0\\
0&0&0&0&1&1\\
\end{matrix}\right)$,
$
\G=\left(\begin{matrix}
1&0&0&0&0&1\\
0&1&0&1&0&0\\
0&0&1&0&1&0\\
\end{matrix}\right)$,
$\mathcal{H}=\left(\begin{matrix}
1&0&0&0&0&0\\
0&1&0&1&0&1\\
0&0&1&0&1&0\\
\end{matrix}\right)$}
\end{center}

Although all four allocations are good, it seems that some are \emph{fairer} than the others.
Indeed, $\mathcal{E}$ looks very unfair because all the resources are allocated to the third agent.
For example allocation $\mathcal{F}$ distributes better the resources than $\mathcal{E}$. If we now focus  only on the subset $\{r_1,r_2,r_3\}$, then $\mathcal{G}$ and $\mathcal{H}$ have a better \emph{local} distribution than $\mathcal{F}$. It is then natural to think that we need a criterion to better distinguish among good allocations. We will develop local criteria in order to distinguish good allocations.
These local criteria will be aggregate in a social criterion.
\end{ejem}

Let  $\mathcal{B}$ the set of good allocations.
For each $r$ in $\R$, we define
\begin{equation}\label{eq-def-A(r)}
A(r)=\{j\in \A: \exists \F=(f_{in} )\in \mathcal{B}\,\, \mbox{such that}\,\,  f_{jr}=1\}
\end{equation}

It is clear that $A(r)\neq \emptyset$. since otherwise, $\mathcal{B}=\emptyset$.
 \Omit{
\begin{remark}
Si $J_{r}=\{i\in\A: \Pref_{ir}=1\}$, entonces $$i\in\A(r)\iff \left\lbrace
\begin{matrix}
i\in\min(\A)& \mbox{si,}\, J_{r}= \emptyset
\\
i\in \max(J_r)& \mbox{si,}\, J_{r}\neq \emptyset
\end{matrix}
\right.$$
\end{remark}
}
Define $\R^*=\{s\in \R: \Pref_{is}=0,\forall i\in \A\}$, that is, the set of resources that nobody requests. Notice that 
 for every  $r\in \R^*$,  $A(r)=\min(\A)$. If $r\not \in \R^*$  and $i,j\in A(r)$ then  $i$ and $j$ request $r$, and moreover both agents have the same hierarchy.

On $\R\setminus \R^* $ we define the equivalence relation $\sim$ as
\begin{equation*}
r\sim s\Leftrightarrow A(r)=A(s).
\end{equation*}

%
We denote by $[r]=\{s\in \R\setminus \R^*: A(s)=A(r) \}$ the equivalence class of $r$ on $\R\setminus \R^*$. Suppose that  there are  $l$ different classes then
\begin{equation}\label{eq:clases}
\R\setminus \R^*=[r_{m_1}]\cup [r_{m_2}]\cup\cdots \cup[r_{m_l}]
\end{equation}

If we put $\R^*=[r_{m_{l+1}}]$, then we have the partition $\R=[r_{m_1}]\cup [r_{m_2}]\cdots [r_{m_l}]\cup [r_{m_{l+1}}] $.


 Let $\F=(f_{in})$ be an allocation  and  $r\in \R$. We will denote as $\F_{[r]}$ the submatrix $A(r)\times [r]$ of  $\F$ composed by the rows with the agents in  $A(r)$ and the columns with the resources in $[r]$.

For every agent   $i\in A(r)$, the number of resources of  $[r]$ allocated to a $i$ via $\F$ is
$\sum_{s\in[r]}f_{is}$.
Consider the  
average of these values:
\begin{equation}\label{eq-media}
\overline{X}(\F_{[r]})=\displaystyle\sum_{i\in A(r)}\displaystyle\sum_{s\in[r]}\dfrac{f_{is}}{|A(r)|}.
\end{equation}
It is easy to see that all good allocations have the same average on the submatrix $A(r)\times [r]$. Moreover,
\begin{equation}\label{eq-average-on-submatrix}
\overline{X}(\F_{[r]})=\frac{|[r]|}{|A(r)|}.
\end{equation}
We consider now a  dispersion  measure with respect to this average denoted  $var(\F_{[r]})$ defined as
\begin{equation}\label{eq-varianza}
var(\F_{[r]})=\displaystyle\sum_{i\in A(r)}\displaystyle\sum_{s\in[r]}\dfrac{(f_{is}-\overline{X}_{\F_{[r]}})^2}{|A(r)|},
\end{equation}
and  we define the dispersion vector of $\F$ by

\begin{equation}\label{eq-vector-variance}
\gamma(F)=(var(\F_{[r_1]}),
var(\F_{[r_2]}),\cdots, var(\F_{[r_{l+1}]}) ).
\end{equation}

Using this vector we define a new relation on the set of good allocations.

\begin{defi}\label{def-satisfaciion-local}
 Let $\F$ and $\G$ good allocations. We say that $\F$ gives more local satisfaction than $\G$, denoted  $\F\succeq_\gamma \G$, if for all $i=1\cdots l+1$ we have  $(\gamma(\G))_i\geq (\gamma(\F))_i$. If for every $\G\in \mathcal{B}$, we have $\F\geq_\gamma \G$, we say that  $\F$ is locally fair.
\end{defi}

We say that
$\F\succ_{\gamma}G$ iff $\F\succeq_{\gamma}\G $ and there exists $k$ such that $(\gamma(G))_k>(\gamma(F))_k$. Whereas
$\F\sim_{\gamma}G$ iff for every $i=1\cdots l+1$ we have $(\gamma(\G))_i= (\gamma(\F))_i$. It is clear that  the relation $\succeq_{\gamma}$  is not a total preorder on  $\mathcal{B}$.
The following example illustrates the previous concepts.

\begin{ejem}\label{ej1}
Let's go back to example \eqref{example-1}  
Notice that, $A(r_{m_1})=A(r_1)=A(r_2)=A(r_3)=\{1,2,3\}$; $A(r_{m_2})=A(r_4)=\{2,3\}$, $A(r_{m_3})=A(r_5)=\{3\}$ and $A(r_{m_4})=A(r_6)=\{1,2,3\}$.
It is also easy to check that
$[r_{m_1}]=\{r_1,\,r_2,\,r_3\}$, $[r_{m_2}]=[r_4]=\{r_4\},\, [r_{m_3}]=[r_5]=\{r_5\}$ and $[r_{m_4}]=[r_6]=\{r_6\}$.
Using equation \eqref{eq-average-on-submatrix}, we have that
in $A(r_{m_1})\times [r_{m_1}]$ the average of any allocation in $\mathcal{B}$ is $1$; in $A(r_{m_2})\times [r_{m_2}]$ the average is $1/2$; in $A(r_{m_3})\times [r_{m_3}]$ the average is $1$; whereas in   $A(r_{m_4})\times [r_{m_4}]$ the average is ${1}/{3}$.

Now by equations \eqref{eq-varianza} and \eqref{eq-vector-variance} we have the dispersion vectors.
\begin{equation*}
\gamma(\mathcal{E})=(2,1/4,0,2/9)\quad\quad
\gamma(\mathcal{F})=(2/3,1/4,0,2/9)
\end{equation*}
\begin{equation*}
\gamma(\mathcal{G})=(0,1/4,0,2/9)\quad\quad
\gamma(\mathcal{H})=(0,1/4,0,2/9)
\end{equation*}

Now, according to Definition \ref{def-satisfaciion-local}, allocations  $\mathcal{G}$ and $\mathcal{H}$ give more local satisfaction than  $\mathcal{F}$ which in turn, gives more local satisfaction than $\mathcal{E}$. That is,
$\G\sim_{\gamma}\mathcal{H}
$ and
$\F\succ_{\gamma}\mathcal{E}$. Moreover, it can be shown that $\G$ and $\mathcal{H}$ are  locally fair allocations.


\end{ejem}


Although the relation $\geq _{\gamma}$ is not total on $\mathcal{B}$, we will show that there always exists a locally fair allocation. We will need the following technical Lemma.

\begin{lema}\label{lema-tecnico}
Let $x_1,\,x_2,\,\cdots,\,x_n$ be positive integers.
If $\displaystyle\sum_{i=1}^n x_i=L$, there exists a natural number $\beta\leq n-1$, such that
\begin{equation}\label{eq-lema-tecnico}
\dfrac{1}{n}\displaystyle\sum_{i=1}^n \left(x_i-\frac{L}{n}\right)^2\geq \dfrac{(n-\beta)\beta}{n^2}
\end{equation}
\end{lema}

\begin{proof}
 For every $1\leq i\leq n$, $x_i\geq 1$, therefore  $L\geq n$. Let $\alpha$ and $\beta$ be the unique natural numbers such that $L=\displaystyle\sum_{i=1}^n x_i=n\alpha +\beta$ with $\beta\in \{0,\cdots, n-1\}$.
Notice that,  for each $i$, $x_i=\alpha$ or $x_i=\alpha+1$. If $x_i=\alpha$ for every $i$, then $L=n\alpha$. consequently, $\beta=0$ and \eqref{eq-lema-tecnico} holds.
Now suppose that there exist $m$ natural numbers such that $x_i=\alpha+1$. If $m=n$ then $\sum_{i=1}^nx_i=n\alpha+n\neq L$. Hence, $m<n$.
On the other hand,
$$L=\displaystyle\sum_{i=1}^mx_i+\displaystyle\sum_{i=m+1}^nx_i=m(\alpha+1)+(n-m)\alpha=n\alpha+m$$
and therefore,
$m=\beta$.
Now,
\begin{eqnarray*}
\dfrac{1}{n}\displaystyle\sum_{i=1}^n\left(x_i-\dfrac{L}{n}\right)^2&=\dfrac{1}{n}\left[\displaystyle \sum_{i=1}^m \left(\alpha+1-\alpha-\dfrac{m}{n}\right)^2+\displaystyle \sum_{i=m+1}^n \left(\alpha-\alpha-\dfrac{m}{n}\right)^2 \right]\\
&=\dfrac{1}{n}\left[\displaystyle \sum_{i=1}^m \left(\dfrac{n-m}{n}\right)^2+\displaystyle \sum_{i=m+1}^n \left(\dfrac{m}{n}\right)^2 \right]\quad\quad\quad\quad\quad \\
&=\dfrac{1}{n}\left[m\dfrac{(n-m)^2}{n^2}+(n-m)\dfrac{m^2}{n^2} \right]= \dfrac{m(n-m)}{n^2}
\end{eqnarray*}
and since $m=\beta$, then equation \eqref{eq-lema-tecnico} holds.
\end{proof}

\begin{teo}\label{main-theorem}
There exists a good allocation $\F^*$ that is locally fair.
\end{teo}

\begin{proof}
We want to find  $\F^*\in \mathcal{B}$ such that $\gamma(\F^*)\leq \gamma(\G)$ for all $\G\in \mathcal{B}$.
Let's consider the partition of $\R$ as in equation \eqref{eq:clases}.

For each $t=1,\dots, l+1$, there exists  unique nonnegative integers $\alpha_{t}$ and $\beta_{t}$  such that
\begin{equation}
|[r_{m_t}]|=|A(r_{m_t})|\alpha_{t}+\beta_{t}
\end{equation}
with $\beta_{t}\in\{0,1,\dots, |A(r_{m_t})|-1\}$.

For each $t\in \{1,\dots l+1\}$, choose  $\beta_{t}$ agents in $A(r_{m_t})$. Let's say $\{j_1,\dots j_{\beta_{t}}\}\subseteq A(r_{m_t})$.

Let's choose $\F^*=(f^*_{in})\in \mathcal{B}$ such that for each $t\in \{1,\dots l+1\}$,
\[\sum_{s\in [r_{m_t}]}f^*_{js}=\begin{cases}
\alpha_{r_{t}}+1, \quad \text{ if } j\in  \{j_1,\dots j_{\beta_{t}}\}\\
\alpha_{r_{t}}, \quad \text{ if } j\in  A(r_{m_t})\setminus \{j_1,\dots j_{\beta_{t}}.\}
\end{cases}\]

Now, on each submatrix  $A(r_{m_t})\times [r_{m_t}]$ we have that
$$\overline{X}(\F^*_{[r_t]})=\alpha_{r_t}+\dfrac{\beta_{r_t}}{|A(r_t)|} \quad \text{and}\quad
var(\F^*_{[r_t]})=\beta_t\dfrac{|A(r_t)|-\beta_t}{|A(r_t)|^2}.$$

On the other hand, for every $\G\in \mathcal{B}$ and $t=1,\dots, l+1$, by equations \eqref{eq-average-on-submatrix} and \eqref{eq-varianza} we have that

$$var(G_{[r_{m_t}]})=\dfrac{1}{|A(r_{m_t})|}\displaystyle\sum_{i=1}^{|A(r)|}\left(\sum_{s\in[r_{m_t}]}g_{is}-\frac{|[r_{m_{t}}]|}{A(r_{m_{t}})}\right)^2
$$

Finally by Lemma \ref{lema-tecnico},
$$var(G_{[r_{m_t}]})\geq var(F^*_{[r_{m_t}]})$$
and consequently, for all $\G\in\mathcal{B}$
 $$\gamma(\F^*)\leq \gamma(\G)$$ \qed

\end{proof}

The following result tells us that for the good allocations, the dispersion measure is invariant under any permutation in the submatrix  $A(r)\times[r]$. For instance,  in Example \ref{ej1}, the allocations $\G$ and $\mathcal{H}$ are different in the column 6. But,
 $$ \G_{[r_{m_4}]}=\left(\begin{matrix}
0&1&0\\
1&0&0\\
0&0&1
\end{matrix}\right)\mathcal{H}_{[r_{m_4}]}.$$ Consequently
the allocation $\F^*$  in the previous theorem is not necessarily unique.

\begin{prop}
Supose that $\R=[r_{m_1}]\cup [r_{m_2}]\cup \cdots \cup [r_{m_l}]\cup [r_{m_{l+1}}]$, where $l$ is the number of different equivalent classes of $\R\setminus \R^*$ and $[r_{m_{l+1}}]=\R^*$. Let $\F$ be a good allocation.
 If $Q_1,\,Q_2\,\cdots Q_{l+1}$ are permutation matrices  of the identity $I_{[r_{m_i}]}$ for every $i=1,\cdots, l+1$, then
$$\gamma(\F)=\gamma(\G)$$
where $\G_{[r_{m_i}]}=Q_i\cdot\F_{[r_{m_i}]}$  for every $i=1,\cdots, l+1$.
\end{prop}

\begin{proof}
$$\gamma(\F)=
(var(\F_{[r_{m_1}]}),var(\F_{[r_2]}),\cdots,var(\F_{[r_{l+1}]}))$$
$$=(var(Q_1\cdot\F_{[r_{m_i}]}),
var(Q_2\cdot\F_{[r_{m_2}]},\cdots,var(Q_{l+1}\cdot\F_{[r_{m_{l+1}}]}))=
\gamma(\G).$$
\end{proof}

\section{Final Remarks}\label{conclu}
In this work we  have tackled the problem of finding a good resource allocation from a qualitative point of view.
A central idea is to use a notion of qualitative social welfare based on the dominance plausible rule of Dubois et al. \cite{DFPP02,DFP03}.
The use of matrix techniques of representation for the relations involved allows us a better organization of results
and a simplification of proofs.

We have generalized the  preliminary results presented by Pino P\'erez et al. \cite{PV12,PVC16} in three ways:
\begin{enumerate}
\item The hierarchical relation between agents can be now a total preorder instead of a linear order.
\item The likelihood relation between the resources can be now a total preorder instead of linear order.
\item The plausibility relation between sets of resources can be now a positive relation extending
the likelihood relation between the resources instead of  only the posibilistic relation considered in \cite{PV12,PVC16}.
\end{enumerate}

Unlike the results in  \cite{PV12,PVC16} in which an optimal allocation is unique, here the class of optimal allocations can contain more than one element. In this direction, we have introduced some local techniques
and give one technique to globally compare  optimal allocations. A future work is to introduce other techniques  for discriminating these elements.  
We note that in this work the attitude of an agent  toward a resource is binary, that is, the agent requests or not the resource.
Thus, another generalization  of this work could be to consider   more complex representations of the  agents' preferences.

%
%
%
%

\bibliographystyle{splncs04}

\end{document}